\documentclass[a4paper,12pt]{article}
\usepackage[T1]{fontenc}    
\usepackage{hyperref}       
\usepackage{url}            
\usepackage{booktabs}       
\usepackage{nicefrac}       
\usepackage{microtype}      

\usepackage{bm}
\usepackage{bbm}
\usepackage[utf8]{inputenc}
\usepackage[english]{babel}
\usepackage{amssymb}
\usepackage{amsmath}
\usepackage{amsthm}
\usepackage{amsfonts}
\usepackage{mathrsfs}
\usepackage{dsfont}
\usepackage{hyperref}
\usepackage{commath}
\usepackage{graphicx}
\usepackage{text comp}
\usepackage{mathtools}
\usepackage{afterpage}
\usepackage{appendix}
\usepackage[linesnumbered,lined,boxed,commentsnumbered]{algorithm2e}

\usepackage{xcolor}

\theoremstyle{plain}
\newtheorem*{thm*}{Theorem}
\newtheorem{thm}{Theorem}[section]

\newtheorem{lemma}[thm]{Lemma}
\newtheorem*{lemma*}{Lemma}
\newtheorem{corollary}[thm]{Corollary}
\newtheorem*{corollary*}{Corollary}

\newtheorem*{prop*}{Proposition}
\newtheorem{defn}{Definition}

\newtheorem*{conjecture*}{Conjecture}

\newcommand{\R}{\mathbb{R}}

\newcommand{\N}{\mathbb{N}}

\newcommand{\1}{\mathbbm{1}}

\newcommand{\dist}{\mathrm{dist}}

\everymath{\displaystyle}

\author{
Henry-Louis de Kergorlay%
    \thanks{%
           School of Mathematics,
           University of Edinburgh,
           Edinburgh, EH9 3FD, UK
           (\texttt{hdekerg@ed.ac.uk})
           }
\and
        Desmond John Higham%
            \thanks{%
           School of Mathematics,
           University of Edinburgh,
           Edinburgh, EH9 3FD, UK
           (\texttt{d.j.higham@ed.ac.uk}).
         Both authors were supported by Engineering and Physical
     Sciences Research Council grant EP/P020720/1.
        }
        }

\date{}
\title{Consistency of Anchor-based Spectral Clustering}

\begin{document}
\maketitle

\begin{abstract}

    Anchor-based techniques reduce the computational
    complexity of spectral clustering algorithms.
    Although empirical tests have shown promising results, there is currently a lack of theoretical support for the anchoring approach.
    We define a specific anchor-based algorithm and show that it
  is amenable to rigorous analysis, as well as being effective in practice.
      We establish the theoretical consistency of the method
       in an asymptotic setting where data is sampled
        from an underlying continuous probability distribution.
         In particular, we provide sharp asymptotic conditions for the algorithm parameters
         which ensure that
          the anchor-based method can recover with high probability
           disjoint clusters that are mutually separated by a positive distance.
            We illustrate the
            performance of the algorithm on synthetic
             data and explain how the
              theoretical convergence analysis can be used to inform the
             practical choice of parameter scalings.
             We also test the accuracy and efficiency of the algorithm on two large scale real data sets.
             We find that the algorithm offers clear advantages over standard spectral clustering.
              We also find that it
             is competitive with the state-of-the-art
             LSC method of Chen and Cai (Twenty-Fifth AAAI Conference on Artificial Intelligence, 2011), while having the added
             benefit of a consistency guarantee.

\end{abstract}

\section{Introduction and Motivation}
\label{sec:int}

Spectral methods are an effective choice for clustering in
unsupervised learning,
\cite{BePa07,spectralClusteringTutorial,ZhRo18}.
Although there are many versions, the overall approach
consists of two distinct steps.
First
the spectrum of a graph Laplacian is used to embed
data points in a low dimensional Euclidean space.
Then a standard clustering method,
typically
$k$-means, is used to
partition the embedded points.

These algorithms may be analyzed
in the context where
there is an underlying continuous domain from which
$n$
data points are generated by sampling identically and independently at random.
In this way, we may quantify the performance of an
algorithm in the limit $n\to \infty$.
In particular,
a learning algorithm is said to be consistent if the quantity, labeling, or other property obtained from the output of the algorithm is shown to converge, in some sense, with high probability (w.h.p.).
For clustering algorithms, it is natural to ask whether
the computed clustering converges to some canonical partition of the underlying sampling set.  In the special case of spectral clustering, consistency may be shown by proving spectral convergence of the graph Laplacian matrices (suitably normalized) to an underlying continuous Laplacian. Such convergence is well-understood in the case where the domain is bounded (or more generally where the domain is a compact Riemannian manifold). In \cite{eigenmaps,diffusionMaps,fromGraphToManifold} Belkin and Nigoyi, Coifman and Lafon, and Singer derived conditions for pointwise convergence. In \cite{consistencySpectralClustering} Luxburg et al.\ first formulated general conditions for spectral convergence to hold. In \cite{variationalApproach} Trillos and Slepcev provided
threshold values for the decrease in the bandwidth parameter of the graph
as a function of $n$, beyond which spectral convergence holds.

Spectral clustering algorithms involve several computationally expensive steps, limiting their use for large scale data sets. First, in order to generate a graph Laplacian matrix, we must compute an affinity matrix, which has time complexity $O(dn^2)$, where $d$ is the dimension of the points, i.e., the number of features. Second, we require the spectrum of a graph Laplacian, which has time complexity $O(n^3)$; see, for instance, \cite{LSC}. Various authors proposed
anchor-based approaches as a means to decrease the overall time complexity without significantly degrading accuracy. These methods start by selecting a small anchor set of points (possibly a subset of the original data points). They then exploit affinities between the whole data set and the anchor set in order to reduce the matrix dimension or to sparsify the affinity matrix.
Such methods include using Nystr\"{o}m approximations, \cite{Nystrom}, and $k$-means based approximate spectral clustering, \cite{KASP}. In \cite{LSC} the authors exploit affinities between anchor points and the original dataset in order to construct a sparsified $n\times n$ affinity matrix.
Since the resulting Laplacian is also sparse,
its spectrum may be computed efficiently.
This LSC algorithm has computation time $O(ndm)$,
where
$m$ is assumed to be fixed and much smaller than $n$,
and hence time complexity is linear in $n$.
 Anchor-based ideas have also been investigated in topological data analysis, particularly in persistent homology, see \cite{witness}. There, affinities between a small set of landmark points and the original data points are exploited in order to build a \textit{witness} complex, to reduce computation time for persistence diagrams.

Despite the promising behavior of anchor-based
spectral clustering methods in practice, to our knowledge there is no complete theory to guarantee their consistency.
Hence, in this work we thoroughly specify a
representative
anchor-based approach, based on random subsampling and nearest neighbor  affinity, and rigorously analyse
its accuracy.
Our key result is Theorem~\ref{consistency anchor spect clust},
showing consistency of the algorithm.
For comparison, we also derive an analogous result in the same setting,
Theorem~\ref{consistency spect clust},
for full spectral clustering.
A key practical insight from our analysis is that
the number of nearest neighbors should be chosen differently,
as a function of $n$,
for the full and anchor-based cases.
Practical computation on synthetic data sets illustrates the
relevance of this scaling information.
We also show that the anchor-based algorithm analyzed here
performs effectively on the large scale MNIST and PenDigits
data sets in comparison with full spectral clustering and LSC.

\section{The AnchorNN Algorithm}\label{sec: anchorNN algo}

Given a vertex set, $X_n:=\{x_1,\dots,x_n\}$, consisting of $n$ points in $\R^d$,
a graph Laplacian may be constructed via an affinity matrix, induced by a geometric graph. There are various popular choices for the geometric graph. The standard (unweighted) geometric graph, studied by Penrose in \cite{randGraphs}, has edges given by $\{\1(|x_i-x_j|<r_n)\ |\ i\neq j\}$, where the bandwidth parameter $r_n>0$ may be allowed to decrease to $0$ as a function of $n$.
Here $| \cdot |$ denotes the Euclidean norm.
A closely related alternative is given by a
$K$ nearest neighbor ($K$NN) graph. Various constructions exist. Here, we connect two vertices by an edge if one of the two vertices is among the $K$ nearest neighbors of the other. The number of nearest neighbors, $K$, is allowed to increase as a function of $n$. We denote by $G(X_n,K)$ the $K$NN graph thus constructed, and let $W$ be the affinity matrix of $G(X_n,K)$; that is,
\[
W_{i,j}:=\1\left(x_i\in N_K(x_j,X_n) \text{~or~}  x_j\in N_K(x_i,X_n)\right),\ i\neq j,
\]
where $N_K(x,X_n)$ denotes the set of $K$ nearest neighbors of $x$ from $X_n$.
The corresponding unnormalized graph Laplacian is defined as
\[
\Delta^{(n)}:=D-W, \quad \text{ where } D :=  \mathrm{diag}
\left(
\sum_{j=1}^nW_{i,j}
\right).
\]
Normalized alternatives of the graph Laplacian may also be used: the random walk Laplacian is given by $D^{-1}\Delta^{(n)}$ and the symmetric Laplacian is given by $D^{-1/2}\Delta^{(n)}D^{-1/2}$.

These versions of the graph Laplacian have been widely studied and their spectra are known to contain important information about the underlying graph. In particular, each has non-negative eigenvalues
with smallest eigenvalue equal to $0$, and the multiplicity of the smallest eigenvalue is equal to the number of connected components of the graph
\cite{Chung1997,spec_hkk,spectralClusteringTutorial}.
Moreover, in the more realistic case where
there are
``approximate'' clusters, rather than fully
disconnected
components,
this structure
can be revealed via ``smooth step functions'' in the
Laplacian eigenvector components
\cite{DXZ01}.
Spectral clustering
algorithms aim to exploit this information.

\begin{algorithm}\label{spect clust algo}
\caption{Full spectral clustering algorithm}

\SetKwInOut{Input}{input}\SetKwInOut{Output}{output}

\Input{$K$, $k$, $X_n$}
\BlankLine

\Output{$C_1^n,\dots,C_k^n$: partition of $X_n$ into $k$ clusters}
\BlankLine

  $W\leftarrow$ affinity $n\times n$ matrix from $X_n$ given by $W_{i,j}:=\1\left(x_i\in N^{(K)}(x_j) \text{ or } x_j\in N^{(K)}(x_i)\right)$\;
  $\Delta^{(n)}\leftarrow$ graph Laplacian from $W$\;
  $u_1,\dots,u_k\leftarrow$ $k$ eigenvectors associated to the $k$ smallest eigenvalues of $\Delta^{(n)}$\;
  $Z_n:=\{z_1,\dots,z_n\}\subset \R^k\leftarrow$ $z_i$ is the $i$th column of $[u_1|\dots |u_k]^T$\;
  $A_1^n,\dots,A_k^n\leftarrow$ partition of $Z_n$ from $k$-means clustering\;
  $ C_1^n,\dots, C_k^n\leftarrow$ partition of $X_n$ given by $C_i^n:=\{x_j\ |\ z_j\in A_i^n\}$\;

\end{algorithm}

Algorithm~\ref{spect clust algo} describes a standard spectral clustering algorithm. Here, the $n$ data points are embedded into a lower dimensional space, using the eigenvectors of the graph Laplacian, then $k$-means clustering is performed on the embedded points.

\begin{algorithm}\label{anchorNN algo}

\SetKwInOut{Input}{input}\SetKwInOut{Output}{output}

\Input{$m$, $K$, $k$, $X_n$}
\BlankLine

\Output{$C_1^m,\dots,C_k^m$: partition of $X_n$ into $k$ clusters}
\BlankLine

  pick uniformly at random $Y_m\subset X_n$\;
  $W\leftarrow$ affinity $m\times m$ matrix from $Y_m$ given by $W_{i,j}:=\1\left(y_i\in N^{(K)}(y_j) \text{ or } y_j\in N^{(K)}(y_i)\right)$\;
  $\Delta^{(m)}\leftarrow$ graph Laplacian from $W$\;
  $ C_1^m,\dots, C_k^m\leftarrow$ partition of $Y_m$ from k-means spectral clustering on $\Delta^{(m)}$\;
  \For{$z\in X_n\setminus Y_m$}{
  $y\leftarrow$ nearest neighbor of $z$ from $Y_m$\;
  $i\leftarrow$ label assigned to $y$\;
  $C_i^m\leftarrow C_i^m\cup \{z\}$\;
  }

 \caption{AnchorNN spectral clustering algorithm}
\end{algorithm}

In contrast, our anchor-based spectral clustering algorithm, AnchorNN, randomly selects a small anchor subset $Y_m\subset X_n$,
computes a partition of $Y_m$ by applying spectral clustering on the Laplacian $\Delta^{(m)}$, then assigns to the remaining points the label given by that of their nearest neighbor in $Y_m$.
 We summarize these steps in Algorithm~\ref{anchorNN algo}.

Randomly selecting the anchor subset in AnchorNN has a cost of $O(1)$.
We then compute a spectral clustering on these anchor points, which costs $O(m^3)$. Finally, we assign to each of the remaining data points the label attributed to its nearest anchor neighbor. This last step costs $O(ndm)$, so AnchorNN is linear in $n$. We note that our simple anchor-based spectral clustering approach is similar to the strategy investigated in \cite{KASP}.
A key difference is that we select the anchor points by random subsampling of the data points, rather than by using a $k$-means algorithm. This facilitates the analysis of the algorithm and allows us to rigorously establish consistency of an anchor-based spectral clustering algorithm for the first time (see Section~\ref{sec:analysis}). Moreover, we are able to provide sharp asymptotic conditions on the various parameters at play in the algorithm ($n$, $m$, $K$) in order to
guarantee consistency, and to compare them with similar sharp conditions for the full spectral clustering algorithm.

The issue of which version of the graph Laplacian should be used
for spectral clustering is discussed, for instance, in
\cite{spec_hkk}
and
\cite{consistencySpectralClustering}.
In \cite{consistencySpectralClustering} the authors argue that the normalized versions should yield better results; however it is assumed there that the bandwidth parameter of the underlying graph is fixed and independent of $n$. If the bandwidth parameter decays to $0$ as $n\to \infty$, then the results found in \cite{errorEstimates,variationalApproach} show that all versions of the graph Laplacian can be used to
 produce consistent full spectral clustering algorithms, with same error convergence.\\

 For concreteness, we focus here on the case of the unnormalized graph Laplacian, but we note that the consistency results in
 Section~\ref{sec:analysis} extend to the two normalized versions of the graph Laplacian, and in our computational experiments we found that
 normalization led to very similar performance of AnchorNN.




\section{Computational Tests}
\label{sec:comp}

Before analyzing the AnchorNN algorithm, we first
test its performance on synthetic and real data sets, and
compare it with full spectral clustering and a state-of-the-art
anchor-based method.
The tests on synthetic data also align closely with the
setting in which we analyze consistency, allowing us to draw on insights from the theory.

\subsection{Adjusted Rand Index}
In order to evaluate the quality of a computed partition of some data points, we choose the Adjusted Rand Index, which is a widely-used assessment of clustering accuracy. The Rand Index is a score in $[0,1]$ which can be interpreted as the probability of two different partitions agreeing on two random data points (\cite{ComparingPartitions,RandIndex}).
In our context we are able to compare a computed
partition with the ground truth, with a larger value indicating better performance.
The Adjusted Rand Index (ARI) is a corrected-for-chance version of the Rand Index, taking into account the potential variability of the number of clusters and of their sizes
\cite{RIvsARI}. See
Appendix~\ref{app: ARI} for a precise definition of the ARI.

\subsection{Synthetic Data}\label{sec:synthetic data}

We compared AnchorNN and full spectral clustering on the six classes
of test data from \cite{syntheticData},
which are illustrated in
Figure~\ref{fig:synth_datasets}.
Here, the data points in 2D are generated probabilistically
according to a known ground truth of nonoverlapping support domains.
These six classes provide a variety of shapes and patterns to test the robustness of clustering algorithms.

\begin{figure}[htp]
    \centering
    \includegraphics[width=15cm]{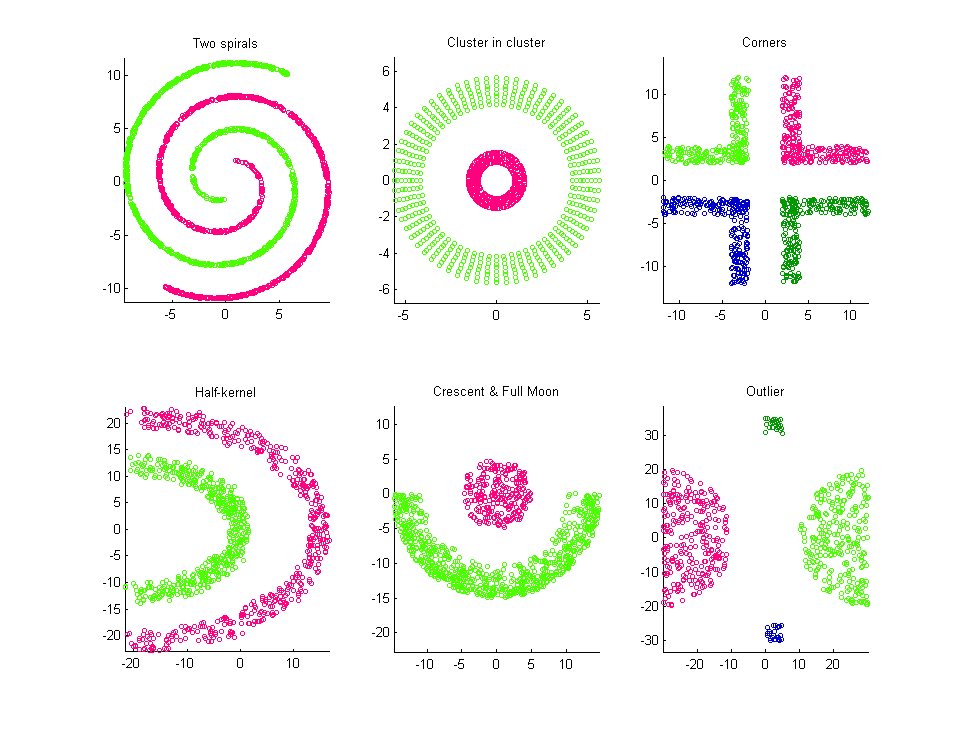}
    \caption{Instances of the six synthetic datasets from
    \cite{syntheticData}.
    }
    \label{fig:synth_datasets}
\end{figure}

We found that
for well-chosen values of $n$ and $K$, the full spectral clustering algorithm yields an ARI of $1$,
that is, full recovery of the known clusters, for all of the six cases. Similarly, for well-chosen $m \ll n$ and $K$, where we recall that
$m$ is the size of the anchor subset, AnchorNN also yields an ARI of $1$.
It is well known that the issue of choosing appropriate values for the parameters $m$ and $K$ is delicate and
data set dependent. It is typically left to the practitioner to fix arbitrary values for $K$ and $m$ without any motivating intuition other than previously reported empirical results (e.g., \cite{LSC,KASP}).
Although not providing a complete answer,
our theoretical results give new insights on how to choose $K$ and $m$
in order to achieve consistency, and how $n$, $m$ and $K$
should relate to each other. In particular,
Theorems~\ref{consistency spect clust} and~\ref{consistency anchor spect clust}
suggest that
full spectral clustering
and
AnchorNN
should be used in different parameter regimes:
AnchorNN will perform better than standard spectral clustering
for small values of $K$, and vice versa for larger values of $K$.
We found that this effect was particularly prominent for the
`Cluster in cluster' data set. We ran both clustering methods on $20$ independent instances of the data set, fixing $n=2000$ and $m=200$, and computed the average ARI. The results confirm that smaller values of $K$ favor AnchorNN and that larger values of $K$ favor spectral clustering;
see Table~\ref{tab: clusterincluster}. Intuitively, for a small value of $K$ and many sampled points, the graph will tend to make local connections and capture local features only, while for a larger value of $K$ or for fewer sampled points, the $K$NN graph will capture more global features. With the `Cluster in cluster' data set, we see
in Figure~\ref{fig:synth_datasets} that a more local scale
exists, made up of small subclusters along the rays composing the main outer cluster, while a more global scale will focus on the the two main clusters.

 \begin{table}[ht]

 \caption{`Cluster in cluster' dataset: ARI comparison for different values of $K$}
 \label{tab: clusterincluster}

\centering
\begin{tabular}[t]{lcc}
\hline
& spectral clustering ($n = 2000$) & AnchorNN ($m = 200$)\\
\hline
$K$ = 8 & 0.31 & 1\\
$K$ = 15 & 1 & 1\\
$K$ = 23 & 1 & 0.22\\
\hline
\end{tabular}

\label{tab:synthetic datasets}
\end{table}







\subsection{Real Data}



We now present results on two standard data sets of handwritten digits,
PenDigits
\cite{96pendigits} and
MNIST \cite{2010mnist}.
PenDigits has $10992$ points of dimension $16$.
MNIST has $70000$ data points of dimension $784$.

Our tests included the LSC algorithm
\cite{LSC}, using an implementation
downloaded from the second author's webpage.
We note that LSC is a state-of-the-art anchor-based algorithm, which
currently is not supported by a consistency result.

Figure~\ref{fig:aris}
shows accuracy results
for full spectral clustering, LSC and AnchorNN,
with number of nearest neighbors $K=7$ and $K=15$.
Here, for each fixed number of anchor points, $m$,
the ARI score is averaged over 20 runs.
Results for full spectral clustering are, of course,
independent of $m$.
The upper two plots correspond to the PenDigits data.
We observe the same phenomenon that was illustrated in
Table~\ref{tab: clusterincluster}, and is explained by our theoretical results in Theorems~\ref{consistency spect clust} and~\ref{consistency anchor spect clust}. With PenDigits for $K=7$, AnchorNN outperforms spectral clustering for $m\in [1000,5000]$ with a peak at around $m=3000$, while for $m>5000$ the performance of AnchorNN degrades and becomes comparable to that of full spectral clustering. On the other hand, for PenDigits with $K=15$  full spectral clustering performs better and its ARI serves as a horizontal asymptotic for anchorNN's ARI, as $m$ increases and approaches $n$.

\begin{figure}[htp]
    \centering
    \includegraphics[width=15cm
    ]{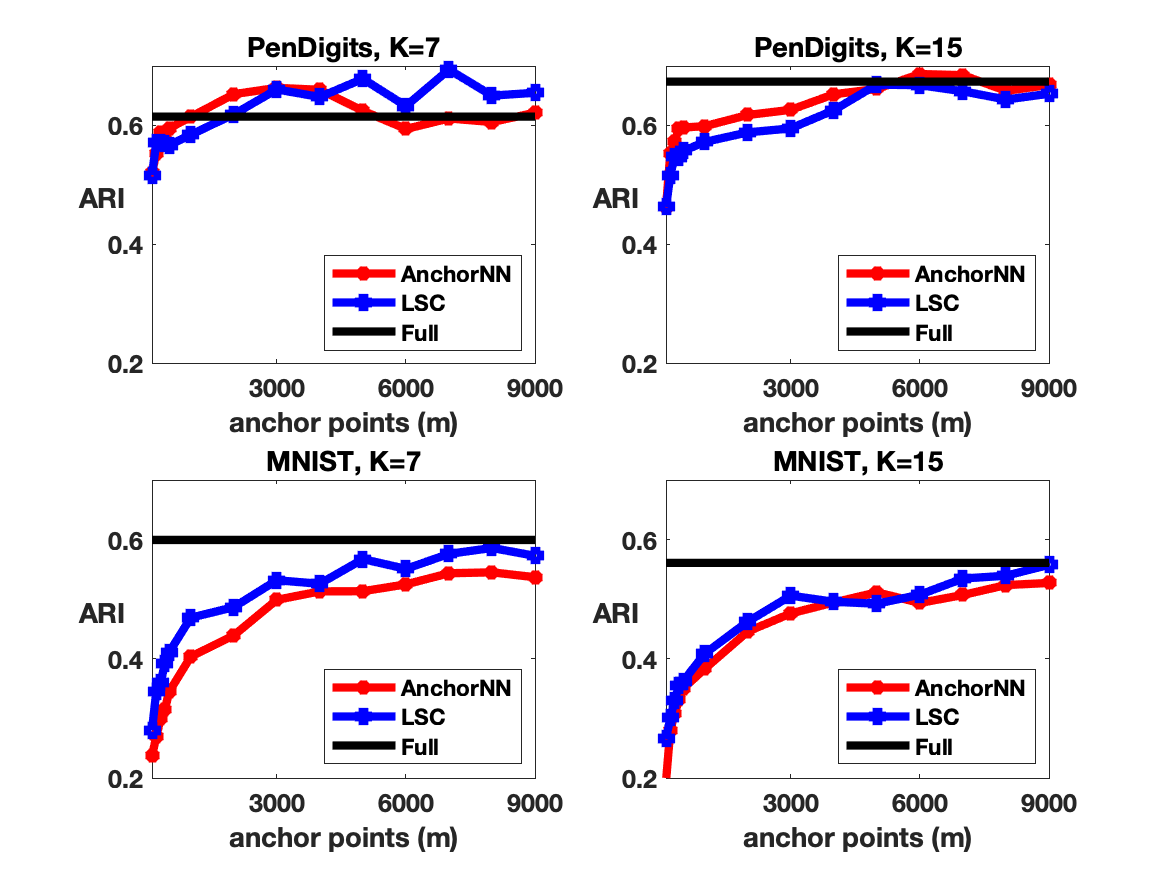}
    \caption{Adjusted Rand Index
    scores (larger is better) for
    full spectral clustering, LSC and AnchorNN
    on PenDigits (upper) and MNIST (lower,)
    using $K=7$ nearest neighbors (left) and
    $K=15$ nearest neighbors (right).
    Horizontal axis gives the number of anchor points,
     $m$, used in each test.
    }
    \label{fig:aris}
\end{figure}

In Figure~\ref{fig:aris} we see that AnchorNN and LSC perform similarly on PenDigits,
with LSC appearing more robust to the
number of nearest neighbors, $K$. Note that for
LSC, $K$ represents the number of nearest neighbors from the anchor subset for a data point in $X_n$, and the final $n\times n$ affinity matrix will have more than $K$ entries per row (see \cite{LSC} for more details on the algorithm). For $K=7$, LSC outperforms the other
two methods for $m>5000$; however this is of limited interest since
at these  large
values of $m$ the time complexity advantage of anchor-based techniques is lost;
see Figure~\ref{fig:times} below.

The lower two plots in
Figure~\ref{fig:aris} show results for MNIST.
 Here, full spectral clustering performs best, with LSC performing slightly better than AnchorNN and
 the ARI of both anchor-based techniques tending to the ARI of full spectral clustering as $m$ increases.

As noted in Section~\ref{sec:int}, both LSC and AnchorNN are linear in $n$.
Figure~\ref{fig:times} shows the computation times
for the results in Figure~\ref{fig:aris}.
As expected,
for small enough $m$
the anchor-based methods can be significantly faster than full spectral clustering. This is especially noticeable with MNIST, which has the larger value of $n$, where full spectral clustering can be several orders of magnitude
slower
than the anchor-based methods.
We also see that
although LSC and AnchorNN
recorded similar computation times,
AnchorNN was consistently faster by a roughly  constant factor.

\begin{figure}[htp]
    \centering
    \includegraphics[width=15cm
    ]{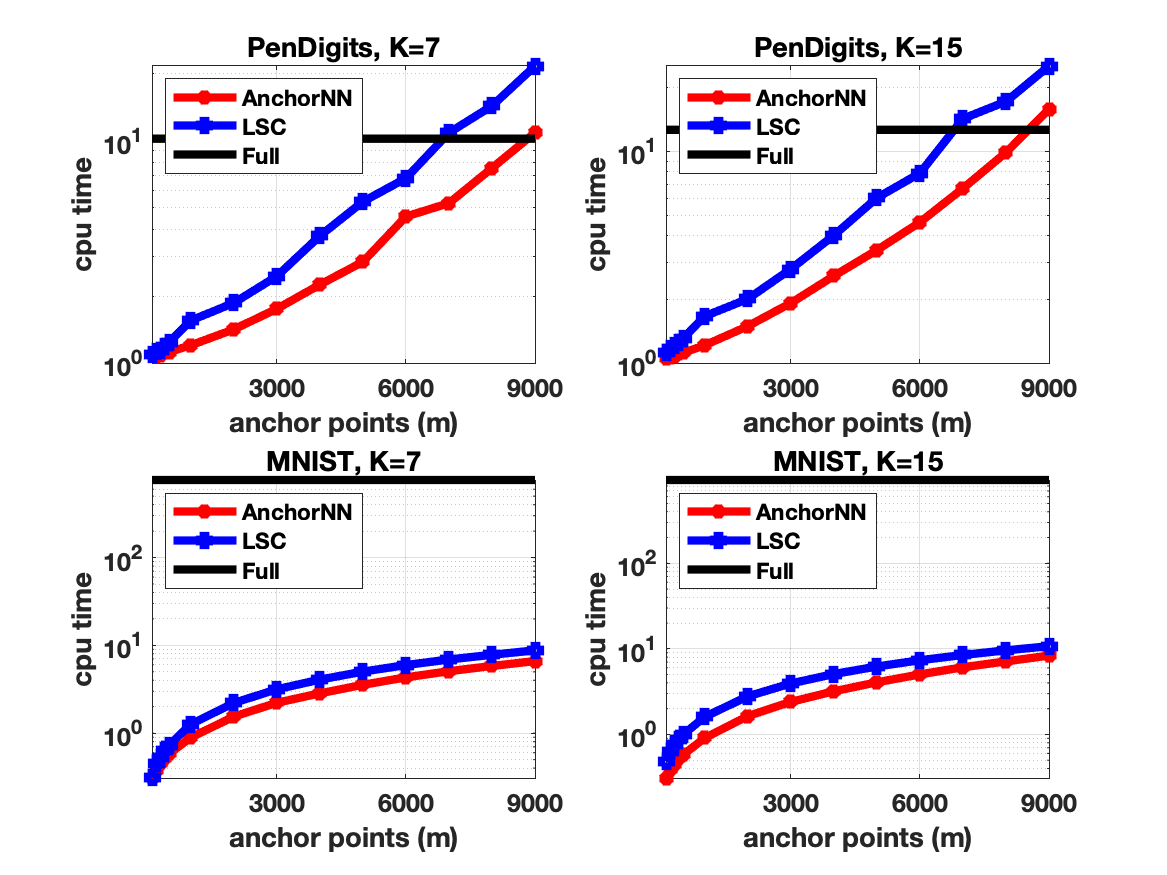}
    \caption{Computing time for
    full spectral clustering, LSC and AnchorNN
    on PenDigits (upper) and MNIST (lower),
    corresponding to the
    results shown in Figure~\ref{fig:aris}.
    }
    \label{fig:times}
\end{figure}

\section{Consistency of AnchorNN spectral clustering}
\label{sec:analysis}

Having established the effectiveness of the
AnchorNN method, we now show that it has a desirable
theoretical consistency property.
Following standard
practice
for analysing spectral clustering methods, we consider
the idealized case where the data is generated from
an underlying sampling set composed of disjoint clusters mutually separated by a positive distance. (If the underlying clusters were overlapping, for instance as in a Gaussian mixture model, then there would not be a single, unambiguous
definition of the correct underlying limiting partition.)
The data sets in Figure~\ref{fig:synth_datasets}
illustrate this scenario.

\subsection{Setting for Consistency Analysis}
We will use the shorthand notation
$
[x]:=[1,x]\cap \mathbb N
$.
Let $d\geq 2$, and let $\nu$ be a probability distribution with sampling density $q\in C^1$ supported on an open bounded set $\Omega\subset \R^d$ with Lipschitz boundary. Let $X_n:=\{x_1,\dots,x_n\}$ be an i.i.d. sample of $n$ random points with respect to $\nu$. Suppose furthermore that $q$ has the classic property
\[
0<q_{\min}\leq q_{\max}<\infty,
\]
where
\[
q_{\min}:=\inf\{q(x)\ |\ x\in \Omega\}
\text{~~and~~}
q_{\max}:=\sup\{q(x)\ |\ x\in \Omega\}.
\]

\begin{defn}
Let $\mathbb P$ be a probability measure and let $(A_n)_{n\in \N}$ be a sequence of  $\mathbb P$-measurable events. We say that $(A_n)_{n\in \N}$ is true \textbf{with high probability}, abbreviated \textbf{w.h.p.}, if
$$
\lim_{n\to \infty}\mathbb P(A_n) = 1.
$$
\end{defn}
With a standard abuse of notation, we will say that $A_n$ is true w.h.p., instead of $(A_n)_{n\in \N}$.\\

\begin{defn}
Let $\mathbb P$ be a probability measure and let $A$ be a $\mathbb P$-measurable event. We say that $A$ is true \textbf{almost surely}, abbreviated \textbf{a.s.}, if
$
\mathbb P(A) = 1
$.
\end{defn}

\begin{defn}
Given a probability measure $\mu$ and a measurable event $A$ such that $\mu(A)>0$, let $\mu|_A$ be the induced probability measure restricted to $A$, given by
$$
\mu|_A(B):=\frac{\mu(B\cap A)}{\mu(A)}.
$$
\end{defn}
Motivated by \cite{variationalApproach}, we define \textit{consistency} as follows.
\begin{defn}[Consistency]\label{defn: consistency}

We say that a sequence of partitions $(\{C_i^n\ |\ i\in [k]\})_{n\in \N}$ of the vertices $(X_n)_{n\in \N}$ is \textbf{consistent}, if there exists a partition $\{C_1,\dots,C_k\}$ of $\Omega$ such that with probability one, up to a subsequence, the following weak convergence of measures holds
$$
\forall\ i\in [k],\ \nu_n|_{C_i^{(n)}}\rightharpoonup \nu|_{C_i},
$$
where
$$
\nu_n:=\frac{1}{\abs{n}}\sum_{x\in X_n}\delta_{x}
$$
denotes the empirical measure induced by $X_n$.
By an abuse of notation, we may say that a partition $\{C_1^n,\dots,C_k^n\}$ is consistent, instead of the sequence $(\{C_i^n\ |\ i\in [k]\})_{n\in \N}$.

Given an algorithm which outputs, for every $n\in \N$, a partition of some points $X_n$ (e.g., Algorithm~\ref{spect clust algo} or Algorithm~\ref{anchorNN algo}), we say that the algorithm is \textbf{consistent} if the sequence of partitions indexed by $n$, obtained in the output, is consistent.
\end{defn}

\subsection{Main Result}

Suppose now that the sampling density $q$ has support given by $\Omega:=\cup_{i\in [k]} C_i$, where
\begin{equation}
\delta:=\min\{\dist(C_i,C_i)\ |\ i\neq j\}>0.
\label{eq:delta}
\end{equation}
Our main theoretical result, showing consistency of the AnchorNN algorithm, appears in
Theorem~\ref{consistency anchor spect clust}.
For comparison,
we first prove the following result concerning the
full spectral clustering algorithm.
\begin{thm}[Consistency of full spectral clustering]\label{consistency spect clust}
There exists $C>0$ depending on $\Omega$ and the dimension $d$, such that if $K\in \N$ satisfies $K=o(n)$ and
$$
K\geq C\log n,
$$
then Algorithm~\ref{spect clust algo} is consistent.
\end{thm}

The type of consistency result
in Theorem~\ref{consistency spect clust}
for full spectral clustering has already been
established for random geometric graphs in several works, see, for instance,  \cite{errorEstimates,variationalApproach,consistencySpectralClustering},
typically under the slightly
less restrictive assumption that the clusters
are not overlapping but connected.
In \cite{improvedKNNspect} the authors also investigate the related $K$NN graph model, where spectral clustering is shown to be consistent if $K=o(n)$ and $K=\omega(\log n)$. These bounds are sharp, in the sense that the connectivity threshold of the underlying $K$NN graph occurs a.s. at $\Theta(\log n)$; see \cite{kNNconnectivity}.\\


In Theorem~\ref{consistency spect clust}
we use the stronger assumption
(\ref{eq:delta}); i.e.,
that the clusters are separated by a positive distance.
Our proof
of Theorem~\ref{consistency spect clust} relies on the observation that, under the assumption (\ref{eq:delta}), the smallest eigenvalue of the graph Laplacian is $0$ with multiplicity $k$, and the associated eigenvectors are given by $\1_{C_i^{(n)}}$, $i\in [k]$, where $C_i^{(n)}=C_i\cap X_n$ and $C_i$ is a cluster.
The introduction of
a positive $\delta$ in
(\ref{eq:delta})
has four advantages.
First,
it allows for a more straightforward proof than those given in the works mentioned above.
Second,
the assumption
allows us to improve the lower bound condition
 $K=\omega(\log n)$ mentioned above to
$K=\Omega(\log n)$.
Third,
and most importantly,
the resulting
analysis can be extended
readily to the
anchor-based case, leading to
Theorem~\ref{consistency anchor spect clust}.
A fourth advantage, which is discussed
in Appendix~\ref{app:further}, is that
the assumption actually implies a stronger, non-asymptotic, type of consistency result.

\begin{thm}[Consistency of AnchorNN]\label{consistency anchor spect clust}
Let $m\to \infty$ as $n\to \infty$. There exists $C>0$ depending on $\Omega$ and $d$, such that if $K\in \N$
satisfies $K=o(m)$ and
$$
K\geq C\log m,
$$
then AnchorNN (i.e., Algorithm~\ref{anchorNN algo}) is consistent.
\end{thm}

This result is, to our knowledge, the first proof of consistency for an anchor-based spectral clustering method. We obtain sharp conditions on the number of nearest neighbors, $K$, as a function of the number of anchor points, $m$, to guarantee the consistency of AnchorNN.
In particular, we show that it is sufficient to have  logarithmic dependency on $m$ in the lower bound on $K$.
The theorem provides a theoretical guarantee that, under appropriate scaling for $K$, our AnchorNN algorithm can
match the recovery property of full spectral clustering, with a much lower time complexity: $O(ndm)$ instead of $O(n^3+n^2d)$, as discussed in Section~\ref{sec: anchorNN algo} and illustrated in Figure~\ref{fig:times}.

\section{Summary}
Anchor-based techniques can dramatically improve the
empirical
performance of spectral clustering, extending the scale at which it can be applied.
The main aim of this work is to add theoretical support
by showing, for the first time, that
an anchor-based method can enjoy the same
consistency property as the full method.
The algorithm that we specify and analyze,
AnchorNN (Algorithm~\ref{anchorNN algo}),
was also found to be
competitive in practice, and the
convergence theory
that we developed helped to shed light on the
practical performance.

The Appendix below contains proofs of
Theorem~\ref{consistency spect clust} and
\ref{consistency anchor spect clust}.

\newpage
\bibliographystyle{plain}
\bibliography{refs}

\begin{thebibliography}{10}

\bibitem{96pendigits}
Fevzi Alimoglu and Ethem Alpaydin.
\newblock Methods of combining multiple classifiers based on different
  representations for pen-based handwriting recognition.
\newblock In {\em Proceedings of the Fifth Turkish Artificial Intelligence and
  Artificial Neural Networks Symposium (TAINN 96)}, 1996.

\bibitem{eigenmaps}
Mikhail Belkin and Partha Niyogi.
\newblock Laplacian eigenmaps and spectral techniques for embedding and
  clustering.
\newblock In {\em Advances in Neural Information Processing Systems},
  volume~14, pages 586 -- 691. 2001.

\bibitem{BePa07}
Mikhail Belkin and Partha Niyogi.
\newblock Convergence of {L}aplacian eigenmaps.
\newblock In {\em Advances in Neural Information Processing Systems},
  volume~19, pages 129--136. MIT Press, 2007.

\bibitem{vanishHomo}
Omer Bobrowski and Shmuel Weinberger.
\newblock On the vanishing of homology in random \v {C}ech complexes.
\newblock {\em {R}andom {S}tructures and {A}lgorithms}, 51(1):14--51, 2016.

\bibitem{kNNconnectivity}
Maria~R. Brito, Edgar~L. Chavez, Adolfo~J. Quiroz, and Joseph~E. Yukich.
\newblock Connectivity of the mutual $k$-nearest-neighbor graph in clustering
  and outlier detection.
\newblock {\em Statistics and Probability Letters}, 35:33 -- 42, 1997.

\bibitem{improvedKNNspect}
Jeff Calder and Nicolas~G. Trillos.
\newblock Improved spectral convergence rates for graph {L}aplacian on
  epsilon-graphs and k-{NN} graphs.
\newblock {\em arXiv:1910.13476}, 2019.

\bibitem{Wthesis}
Wei Chai.
\newblock {\em Random Topological Structures}.
\newblock PhD thesis. University of Chicago, 2018.

\bibitem{Nystrom}
Wen-Yen Chen, Yangqui Song, Hongjie Bai, Chih-Jen Lin, and Edward~Y. Chang.
\newblock Parallel spectral clustering in distributed systems.
\newblock {\em IEEE Transactions on Pattern Analysis and Machine Intelligence},
  2010.

\bibitem{LSC}
Xinlei Chen and Deng Cai.
\newblock Large scale spectral clustering with landmark-based representation.
\newblock {\em Proceedings of the Twenty-Fifth AAAI Conference on Artificial
  Intelligence}, 2011.

\bibitem{Chung1997}
Fan R.~K. Chung.
\newblock {\em Spectral Graph Theory}.
\newblock American Mathematical Society, Rhode Island, 1997.

\bibitem{diffusionMaps}
Ronald~R. Coifman and Stéphane Lafon.
\newblock Diffusion maps.
\newblock {\em Applied and Computational Harmonic Analysis}, 21(1):5 -- 30,
  2006.

\bibitem{witness}
Vin de~Silva and Gunnar Carlsson.
\newblock Topological estimation using witness complexes.
\newblock {\em Eurographics Symposium on Point-Based Graphics}, pages 157--166,
  2004.

\bibitem{DXZ01}
Chris Ding, Xiaofeng He, and Hongyuan Zha.
\newblock A spectral method to separate disconnected and nearly-disconnected
  web graph component.
\newblock {\em Proceedings of the Seventh ACM SIGKDD International Conference
  on Knowledge Discovery and Data Mining}, pages 275--280, 2001.

\bibitem{randCoverings}
Leopold Flatto and Donald~J. Newman.
\newblock Random coverings.
\newblock {\em Acta Math.}, 138:241--264, 1977.

\bibitem{sphereCoverage}
Peter Hall.
\newblock On the coverage of $k$-dimensional space by $k$-dimensional spheres.
\newblock {\em Annals of Probability}, 1985.

\bibitem{spec_hkk}
Desmond~J. Higham, Gabriela Kalna, and Milla~J. Kibble.
\newblock Spectral clustering and its use in bioinformatics.
\newblock {\em J. Computational and Applied Math.}, 204:25--37, 2007.

\bibitem{ComparingPartitions}
Lawrence Hubert and Phipps Arabie.
\newblock Comparing partitions.
\newblock {\em Journal of Classification}, 2(1):193 -- 218, 1985.

\bibitem{cechBoundary}
Henry-Louis~de Kergorlay, Ulrike Tillmann, and Oliver Vipond.
\newblock Random \v {C}ech complexes on manifolds with boundary.
\newblock {\em arXiv:1906.07626}, 2019.

\bibitem{syntheticData}
Jeroen Kools.
\newblock 6 functions for generating artificial datasets
  (https://www.mathworks.com/matlabcentral/fileexchange/41459-6-functions-for-generating-artificial-datasets).
\newblock {\em MATLAB Central File Exchange.}, 2020.

\bibitem{2010mnist}
Yann LeCun, Corinna Cortes, and Christopher J.~C. Burges.
\newblock {MNIST} handwritten digit database.
\newblock {\em ATT Labs [Online]}, 2010.

\bibitem{randomCirclesonSphere}
Patrick~A.P. Moran and Stephen~Fazekas de~St~Groth.
\newblock Random circles on a sphere.
\newblock {\em Biometrika}, pages 384--396, 1962.

\bibitem{randGraphs}
Mathew~D. Penrose.
\newblock {\em Random Geometric Graphs}, volume~5 of {\em Oxford Studies in
  Probability}.
\newblock Oxford University Press, Oxford, 2003.

\bibitem{RandIndex}
William~M. Rand.
\newblock Objective criteria for the evaluation of clustering methods.
\newblock {\em Journal of the American Statistical Association}, 66(336):846 --
  850, 1971.

\bibitem{fromGraphToManifold}
Amit Singer.
\newblock From graph to manifold {L}aplacian: the convergence rate.
\newblock {\em Applied and Computational Harmonic Analysis}, 21(1):128--134,
  2006.

\bibitem{errorEstimates}
Nicolas~G. Trillos, Moritrz Gerlach, Matthias Hein, and Dejan Slepcev.
\newblock Error estimates for spectral convergence of the graph {L}aplacian on
  random geometric graphs towards the {L}aplace--{B}eltrami operator.
\newblock {\em Foundations of Computational Mathematics}, 2019.

\bibitem{variationalApproach}
Nicolas~G. Trillos and Dejan Slepcev.
\newblock A variational approach to the consistency of spectral clustering.
\newblock {\em Applied and Computational Harmonic Analysis}, 45(2):239 -- 281,
  2018.

\bibitem{RIvsARI}
Nguyen~Xuan Vinh, Julien Epps, and James Bailey.
\newblock Information theoretic measures for clustering comparison: Is a
  correction for chance necessary.
\newblock {\em ICML'09: Proceedings of the 26th Annual International Conference
  on Machine Learning}, 2009.

\bibitem{spectralClusteringTutorial}
Ulrike von Luxburg.
\newblock A tutorial on spectral clustering.
\newblock {\em Statistics and Computing}, 17(4):395 -- 416, 2007.

\bibitem{consistencySpectralClustering}
Ulrike von Luxburg, Mikhail Belkin, and Olivier Bousquet.
\newblock Consitency of spectral clustering.
\newblock {\em The Annals of Statistics}, 36(2):555 -- 586, 2008.

\bibitem{KASP}
Donghui Yan, Ling Huang, and Michael~I. Jordan.
\newblock Fast approximate spectral clustering.
\newblock {\em Proceedings of the 15th ACM International Conference on
  Knowledge Discovery and Data Mining}, 2009.

\bibitem{ZhRo18}
Yilin Zhang and Karl Rohe.
\newblock Understanding regularized spectral clustering via graph conductance.
\newblock In {\em Advances in Neural Information Processing Systems 31}, pages
  10631--10640. 2018.

\end{thebibliography}
\newpage
\begin{appendices}
\section{Adjusted Rand Index}\label{app: ARI}
Given a pair of data points and two (possibly) different partitions of the data points, we say that the pair is \textbf{agreeing} if they are in the same cluster for both partitions or in a different cluster for both partitions;
otherwise we say that the pair is \textbf{disagreeing}.
The Rand Index is the ratio of the number of agreeing pairs to the total number of pairs of points
\cite{ComparingPartitions,RandIndex}.
This yields a score in $[0,1]$ which can be interpreted as the probability that the two different partitions will agree on two random points. The Adjusted Rand Index (ARI),
\cite{RIvsARI},
is a corrected-for-chance version of the above Rand Index, which aims to take into account the variability of the number of clusters and of the sizes of those clusters. Given two partitions $\mathcal C:=\{C_1,\dots,C_r\}$ and $\mathcal C':=\{C_1',\dots,C_s'\}$, for every $(i,j)\in [r]\times [s]$ let
$$
n_{i,j}:=\abs{C_i\cap C_j'},
$$
and let
$$
a_i:=\sum_{j=1}^s n_{i,j}
$$
and
$$
b_j:=\sum_{i=1}^r n_{i,j}.
$$
The ARI between $\mathcal C$ and $\mathcal C'$ is then defined as
$$
ARI(\mathcal C,\mathcal C'):=\frac{\sum_{i,j}{n_{i,j}\choose 2}-\big[\sum_{i,j}{a_i\choose 2}{b_j\choose 2}\big]/{n\choose 2}}{(1/2)\big[\sum_i{a_i\choose 2}+\sum_j{b_j\choose 2}\big]-\big[\sum_{i,j}{a_i\choose 2}{b_j\choose 2}\big]/{n\choose 2}}.
$$

\section{Proofs of Theorems~\ref{consistency spect clust} and~\ref{consistency anchor spect clust}}\label{app: proofs}
Our proofs of Theorems~\ref{consistency spect clust} and~\ref{consistency anchor spect clust}
rely on concentration and random covering results.\\

\subsection{Preliminary results}
First recall some classic Chernoff-type bounds (see for instance Lemma $1.1$ in \cite{randGraphs}) for binomial random variables.
Let $n\in \N$, let $p\in [0,1]$, and let
$X\sim \textrm{Bi}(n,p)$ be a binomial random variable with parameters $n$ and $p$. For $x>0$ let
\begin{equation}\label{H defn}
H(x):=1-x+x\log x,
\end{equation}
and set $H(0):=1$.
We shall need the following Chernoff-type bounds
\begin{equation}\label{Chernoff bounds}
\begin{cases}
    \mathbb P\left(X\geq k\right)&\leq \exp\left(-np H\left(\frac{k}{np}\right)\right)\text{, }
     \qquad k\geq np,\\
    \mathbb P\left(X\leq k\right)&\leq \exp\left(-npH\left(\frac{k}{np}\right)\right)\text{, }
     \qquad k\leq np.
    \end{cases}
\end{equation}

We now proceed via a sequence of lemmas. We
first derive a result about random coverings of bounded subsets of $\R^d$. What is the minimum value of $r$ such that w.h.p.
$$
\Omega\subset \cup_{x\in X_n}B(x,r)?
$$
Several papers have investigated this question (e.g., \cite{randCoverings,sphereCoverage,randomCirclesonSphere,vanishHomo,Wthesis}). For our purposes, we need lower bound conditions on $r_n$ such that this random covering occurs a.s.\ for all sufficiently large $n$.

\begin{lemma}\label{random covering lemma}
There exists $C>0$ depending on $\Omega$ and $d$, such that if
$$
nq_{\min} r_n^d \geq C\log n,
$$
then a.s., there exists  $n_0\in \N$, such that for all $n\geq n_0$
$$
\Omega \subset \cup_{x\in X_n}B(x,r_n).
$$
\end{lemma}

\begin{proof}
Divide $\R^d$ into a grid of small cubes $\{Q_{i,n}|i\in \N\}$ of width $\gamma r_n$, where $\gamma>0$ is to be determined, and define
$$
S_n:=\{i\in \N\ |\ Q_{i,n}\subset \Omega\}.
$$
Since $\Omega$ is bounded and has Lipschitz boundary, we can choose $\gamma$ so small that the following inclusion of probability events holds
$$
\{\forall\ i\in S_n,\ Q_{i,n}\cap X_n\neq \emptyset\} \subset \{\Omega\subset \cup_{x\in X_n}B(x,r_n)\}.
$$

Thus
\begin{align*}
\mathbb P\left(\Omega \not\subset \cup_{x\in X_n}B(x,r_n)\right)&\leq \mathbb P\left(\exists\ i\in S_n,\ Q_{i,n}\cap X_n=\emptyset \right)\\
&\leq \sum_{i\in S_n}\mathbb P\left(\abs{Q_{i,n}\cap X_n}=0\right).
\end{align*}
Since $\Omega$ has bounded diameter, we have the estimate
$$
\abs{S_n}\lesssim r_n^{-d}\leq n,
$$
and for all $i\in S_n$, using the Chernoff bounds in (\ref{Chernoff bounds}),
\begin{align*}
\mathbb P\left(\abs{Q_{i,n}\cap X_n}=0\right)&\leq \exp\left(-n\nu(Q_{i,n})\right)\\
&\leq \exp\left(-nq_{\min}\gamma^d r^d\right).
\end{align*}
Suppose now that
$$
nq_{\min}\gamma^dr_n^d\geq 3\log n,
$$
then
$$
\mathbb P\left(\Omega\not\subset \cup_{x\in X_n}B(x,r_n)\right)\lesssim n^{-2}.
$$
By the Borel-Cantelli lemma, we deduce that there exists a.s. $n_0\in \N$ such that for all $n\geq n_0$
$$
\Omega\subset \cup_{x\in X_n}B(x,r_n).
$$

\end{proof}
We also use the following lemma on the connectivity of random $K$NN graphs. It can be deduced from the above lemma, or directly from the results in \cite{kNNconnectivity}, which derived threshold values
on $K$ for the connectivity of a random $K$NN graph, provided that the sampling domain is \textit{grid compatible} (see definition in \cite{kNNconnectivity}).
We note that $\Omega$ in our context, being bounded with Lispchitz boundary, is grid compatible.

\begin{lemma}[\cite{kNNconnectivity}]\label{lemma: kNNconnectivity}
Let $\Omega'\subset \Omega$ be connected with Lipschitz boundary. There exists $C>0$ depending on $\Omega$ and $d$, such that if
$$
K\geq C\log n,
$$
then, almost surely, there exists $n_0\in \N$, such that for all $n\geq n_0$ the graph $G(X_n\cap \Omega',K)$ is connected.
\end{lemma}
In our setting $\Omega$ is not connected, but we will apply Lemma~\ref{lemma: kNNconnectivity} to  the clusters $C_i\subset \Omega$, $i\in [k]$.\\

We also need to estimate the measure of a small ball near the boundary $\partial \Omega$. This type of estimate
has been investigated in the more general case where $\Omega$ is a Riemannian manifold with boundary in \cite{Wthesis,cechBoundary}.  For our purposes, the following lower bound suffices.
\begin{lemma}\label{lemma: half ball}
Suppose that $r_n=o(1)$, then for $n$ sufficiently large and all $x\in \Omega$
$$
\nu(B(x,r))\geq \frac{1}{2}q_{\min}\omega_d r^d,
$$
where $\omega_d$ denotes the volume of a unit ball in $\R^d$.
\end{lemma}
\begin{proof}
Since $\Omega$ is bounded with Lisphitz boundary, then $\partial \Omega$ has bounded curvature and we can pick $n$ sufficiently large ($r$ sufficiently small) so that for all $x\in \Omega$, $B(x,r)\cap \Omega$ contains a half-ball from $B(x,r)$ (note that this is trivial if $\dist(x,\partial \Omega)>r$). We then have, for all $x\in \Omega$
\begin{align*}
    \nu(B(x,r))&\geq q_{\min}\mathcal L(B(x,r)\cap \Omega)\\
    &\geq \frac{1}{2}q_{\min}\mathcal L(B(x,r))\\
    &\geq \frac{1}{2}q_{\min}\omega_d r^d,
\end{align*}
where $\mathcal L$ denotes the Lebesgue measure on $\R^d$.

\end{proof}

We may now show the following discrepancy-type result, estimating the regularity of the empirical measure with respect to the underlying probability measure $\nu$ on each random ball. We note that, for the sake of  readability, we did not state the results below with the sharpest possible multiplicative constants.
\begin{lemma}\label{concentration lemma}
If $r_n=o(1)$ and
$$
nq_{\min}\omega_d r_n^d\geq 9\log n,
$$
then there exists a.s. $n_0\in \N$ such that for all $n\geq n_0$ and all $x\in X_n$
$$
\frac{n}{2}\nu(B(x,r))\leq \abs{X_n\cap B(x,r)}.
$$
\end{lemma}
\begin{proof}

For $x\in X_n$, let $X_n^{(x)}:=X_n\setminus \{x\}$ and note that $\abs{X_n^{(x)}\cap B(x,r)}$ is a binomial random variable with parameters $n-1$ and $\nu(B(x,r))$.\\

Let $\epsilon>0$ be sufficiently small that
$$
H(\epsilon)\geq 3/4,
$$
where $H$ is defined in (\ref{H defn}). We have by the Chernoff bounds in (\ref{Chernoff bounds}) and Lemma~\ref{lemma: half ball}
$$
\mathbb P\left(\abs{X_n^{(x)}\cap B(x,r)}<(1-\epsilon)(n-1)\nu(B(x,r))\right)\leq \exp\left(-(1/2)(n-1)q_{\min}\omega_d r^dH(\epsilon)\right),
$$
and for $n$ sufficiently large
$$
(1/2)H(\epsilon)(n-1)\geq n/3,
$$
and thus
$$
\mathbb P\left(\abs{X_n^{(x)}\cap B(x,r)}<(1-\epsilon)n\nu(B(x,r))\right)=O(n^{-3}).
$$
 This holds for every $x\in X_n$, hence by a union bound
$$
\mathbb P\left(\exists\ x\in X_n,\ \abs{X_n^{(x)}\cap B(x,r)}<(1-\epsilon)n\nu(B(x,r))\right) = O(n^{-2}),
$$
and we deduce by the Borel-Cantelli lemma that there exists a.s.\ $n_0\in \N$ such that for all $n\geq n_0$, for all $x\in X_n$
$$
(1-\epsilon)(n-1)\nu(B(x,r))\leq \abs{X_n^{(x)}\cap B(x,r)}
\Rightarrow\ \frac{n}{2}\nu(B(x,r))\leq \abs{X_n\cap B(x,r)},
$$
for $n$ sufficiently large and $\epsilon$ sufficiently small.

\end{proof}

\begin{corollary}\label{kNN control}
If $K\in \N$ satisfies $K=o(n)$ and
$$
K\geq \frac{9}{2}\log n,
$$
and $r$ is such that
 $$
 \frac{n}{2}q_{\min}\omega_d r^d:=K,
 $$
 then a.s., there exists $n_0\in \N$, such that for all $n\geq n_0$ and all $x,y\in X_n$

$$
\1
\left(y\in N_K(x,X_n) \text{ or }x\in N_K(y,X_n)\right)\leq \1\left(\abs{x-y}<  r\right).
$$
\end{corollary}

\begin{proof}
By the assumptions on $K$ and $r$, for all $x\in X_n$
$$
K\leq \frac{n}{2}\nu(B(x,r)),
$$
$r=o(1)$ and
$$
nq_{\min}\omega_d r^d\geq 9\log n.
$$
By Lemma~\ref{concentration lemma}, we then deduce that
$$
K\leq \abs{X_n\cap B(x,r)}.
$$
It follows that $y\in N_K(x)\Rightarrow y\in B(x,r)$, and likewise $x\in N_K(y)\Rightarrow x\in B(y,r)$. Hence
$$
\1\left(y\in N_K(x)\text{ or }x\in N_K(y)\right)\leq \1\big(|x-y|<r\big)
$$
for all $x,y\in X_n$.

\end{proof}

\subsection{Proofs of the theorems}
Theorem~\ref{consistency spect clust} may now be proved as follows.
\begin{proof}[Proof of Theorem~\ref{consistency spect clust}]
For $i\in [k]$, define
$$
C_i^{(n)}:=C_i\cap X_n.
$$
Let $K\geq C\log n$, where $C>0$ and is sufficiently large that, by Lemma~\ref{lemma: kNNconnectivity}, there exists a.s.\ $n_0\in \N$, such that for all $n\geq n_0$ and all $i\in [k]$, the graph $G(C_i^{(n)},K)$ is connected. Choose also $C>9/2$ and let $r$ be chosen as in Corollary~\ref{kNN control}, such that a.s.\ for all $n$ sufficiently large and all $x,y\in X_n$
$$
\1\left(y\in N_K(x)\text{ or }x\in N_K(y)\right)\leq \1\left(|x-y| < r\right).
$$
Since $K=o(n)\Rightarrow r=o(1)$, for $n$ sufficiently large and all $(x,y)\in C_i\times C_j$ such that $i\neq j$ in $[k]$,
$$
|x-y|\geq \delta >r.
$$
Hence
$$
x\not\in N_K(y)\text{ and }y\not\in N_K(x)
$$
and the vertices $x$ and $y$ are not connected by an edge.
This shows that $G(X_n,K)=\cup_{i\in [k]}G(C_i^{(n)},K)$.
Hence the unnormalized graph Laplacian matrix  $\Delta^{(n)}$ is block diagonal with $k$ blocks induced by $\1_{C_i^{(n)}},\ i\in [k]$. By Proposition $2$ in \cite{spectralClusteringTutorial}, we deduce that
$$
\mathrm{ker}(\Delta^{(n)}) =
\mathrm{span}\{\1_{C_i^{(n)}}\ |\ i\in [k]\},
$$
 and the $k$ eigenvectors of $\Delta^{(n)}$ with eigenvalue $0$ are given by $\1_{C_i^{(n)}},\ i\in [k]$. Hence we have $\forall\ i\in [k],\ u_i:=\1_{C_i^{(n)}}$ in line $4$ of Algorithm $1$, and we see that the output partition is given exactly by $C_1^{(n)},\dots,C_k^{(n)}$. However, by the law of large numbers $$\forall\ i\in [k],\ \nu_n|_{C_i^{(n)}}\rightharpoonup \nu|_{C_i}.$$

\end{proof}


Using Theorem~\ref{consistency spect clust} and Lemma~\ref{random covering lemma} on random coverings of bounded domains, we prove
Theorem~\ref{consistency anchor spect clust} as follows.
\begin{proof}[Proof of Theorem~\ref{consistency anchor spect clust}]
Let $Y_m\subset X_n$ be a random anchor subset, with $m\to \infty$ as $n\to \infty$. Define for $i\in [k]$
$$\tilde C_i^{(m)}:=C_i\cap Y_m,$$
and define the label map
\begin{align*}
    \varphi:Y_m&\longrightarrow [k]\\
    y&\longmapsto \sum_{i\in [k]}i\1\big(y\in \tilde C_i^{(m)}\big).
\end{align*}
For $x\in X_n$, let $y(x)$ be the nearest neighbor of $x$ from $Y_m$. For $i\in [k]$ define
$$
C_i^{(m)} := \{x\in X_n\ |\ \varphi(y(x))=i\}.
$$
Let $K\geq C\log m$, where $C>0$ is such that Theorem~\ref{consistency spect clust} holds, replacing $X_n$ by $Y_m$.\\

By the proof of Theorem~\ref{consistency spect clust}, a.s. for all $n$ sufficiently large the output partition in line $4$ of Algorithm $2$ is given by
$$
\tilde C^{(m)}_1,\dots,\tilde C^{(m)}_k.
$$
Hence the final output partition of the algorithm is
$$
C_1^{(m)},\dots,C_k^{(m)}.
$$
Let $r$ be such that
$$
\frac{m}{2}q_{\min}\omega_d r^d:=K,
$$
and choose $C>0$ also such that Lemma~\ref{random covering lemma} holds for all $C_i,\ i\in [k]$. Thus a.s. for all $n$ sufficiently large
$$
\forall i\in [k],\ C_i\subset \cup_{x\in \tilde C_i^{(m)}}B(x,r).
$$
Since $K=o(m)$,
$$
|x-y(x)|<r = o(1)<\delta,\text{ for  $n$ sufficiently large.}
$$
Hence
\begin{align*}
    x\in C_i\cap X_n&\Leftrightarrow y(x)\in \tilde C_i^{(m)}\\
    &\Leftrightarrow x\in C_i^{(m)}.
\end{align*}
This shows that a.s.\ there exists $n_0\in \N$ such that for all $n\geq n_0$
$$
\forall\ i\in [k],\ C_i^{(n)} = C_i^{(m)}.
$$
Hence the output partition of Algorithm $2$ is a.s., for all $n$ sufficiently large,
$$
C_1^{(n)},\dots,C_k^{(n)},
$$
which is a consistent partition by the law of large numbers, as in the proof of Theorem~\ref{consistency spect clust}.

\end{proof}
\section{Further comments on consistency}\label{app:further}

We finish by a note on consistency and our proofs of Theorems~\ref{consistency spect clust} and~\ref{consistency anchor spect clust}.\\

First, we remark that Definition~\ref{defn: consistency} could be rephrased as follows.

\begin{defn}[Consistency alternative]\label{defn: consistency alternative}
Let $X_n$ be an i.i.d.\ sample of $n$ points with respect to a probability measure $\nu$ supported on $\Omega$, and let $(\{C_i^n\ |\ i\in [k]\})_{n\in \N}$ be a sequence of partitions of the $(X_n)_{n\in \N}$. We say that the sequence is \textbf{consistent} if there exists a partition $\{C_i\ |\ i\in [k]\}$ of $\Omega$, such that up to a subsequence,
$$
\forall\ i\in [k],\ \lim_{n\to \infty}\pi(\nu_n|_{C_i^n},\nu_n|_{C_i^{(n)}})= 0,
$$
where $\pi$ denotes the L\'evy-Prokhorov metric on the space $\mathcal P(\R^d)$ of probability measures on $\R^d$.
\end{defn}
Since weak convergence of probability measures is equivalent to convergence in the L\'evy-Prokhorov metric, $\R^d$ being separable, we have by the law of large numbers
$$
\lim_{n\to \infty}\pi(\nu_n|_{C_i^{(n)}},\nu|_{C_i})= 0,
$$
hence by the triangle inequality
\begin{align*}
    \lim_{n\to \infty}\pi(\nu_n|_{C_i^n},\nu_n|_{C_i^{(n)}})= 0&\Leftrightarrow \lim_{n\to \infty}\pi(\nu_n|_{C_i^n},\nu|_{C_i})= 0\\
    &\Leftrightarrow \nu_n|_{C_i^n}\rightharpoonup \nu|_{C_i},
\end{align*}
which shows that Definition~\ref{defn: consistency alternative} is indeed equivalent to Definition~\ref{defn: consistency}.\\

In other words, our definition of consistency asks that the output partition $\{C_i^n\ |\ i\in [k]\} $ of the algorithm tends to the true partition $\{C_i^{(n)}\ |\ i\in [k]\}$ of the data set as $n\to \infty$. This is an asymptotic requirement. We note that on the other hand, the proofs of Theorems~\ref{consistency spect clust} and~\ref{consistency anchor spect clust} above actually show that a.s., if $n$ is sufficiently large, then the output partitions of Algorithms~\ref{spect clust algo} and~\ref{anchorNN algo} are equal to the true partition $\{C_i^{(n)}\ |\ i\in [k]\}$ of $X_n$.\\

We emphasize that the non-asymptotic nature of our proofs relies on the simplifying assumption~(\ref{eq:delta});
the underlying clusters are mutually separated by a positive distance. If this were not the case, then
the sequence of arguments used in the proofs would no longer be valid.

Since the synthetic data sets in Section~\ref{sec:synthetic data} satisfy the assumption~(\ref{eq:delta}), the proofs of Theorems~\ref{consistency spect clust} and~\ref{consistency anchor spect clust} show that a.s., if $n$ or $m$ are sufficiently large and $K$ suitably chosen then the clusters obtained from Algorithms~\ref{spect clust algo} and~\ref{anchorNN algo}
are
true partitions that hence give
an ARI of $1$.
This completely explains the
behavior that we observed in our experiments.
\end{appendices}

\end{document}